%% file: main.tex
\documentclass{article} 
\usepackage{iclr2026_conference,times}

\usepackage{hyperref}
\usepackage{url}

\input{preamble}
\usepackage{wrapfig}

\input{math_commands.tex}

\title{Post-hoc Self-explanation of CNNs}

\author{Ahcène Boubekki \& Line H. Clemmensen \\
Department of Mathematical Sciences \\
University of Copenhagen, Denmark \\
\texttt{\{ahcene.boubekki,lkhc\}@math.ku.dk}}

\iclrfinalcopy
\renewcommand{\headrulewidth}{0pt}

\begin{document}

\maketitle

\begin{abstract}
{
Although standard Convolutional Neural Networks (CNNs) can be mathematically reinterpreted as Self-Explainable Models (SEMs), their built-in prototypes do not on their own accurately represent the data. Replacing the final linear layer with a $k$-means-based classifier addresses this limitation without compromising performance. This work introduces a common formalization of $k$-means-based post-hoc explanations for the classifier, the encoder's final output (B4), and combinations of intermediate feature activations. The latter approach leverages the spatial consistency of convolutional receptive fields to generate concept-based explanation maps, which are supported by gradient-free feature attribution maps. Empirical evaluation with a ResNet34 shows that using shallower, less compressed feature activations, such as those from the last three blocks (B234), results in a trade-off between semantic fidelity and a slight reduction in predictive performance.}
\end{abstract}

\section{Introduction}

Convolutional Neural Networks (CNNs) are the basis for several foundation models, particularly for image classification~\cite{imagenet}, but are also utilized for other data types~\cite{ribeiro_automatic_2020}. Their architectures, inspired by the biological visual cortex, typically consist of a feature extractor or encoder that combines 2D convolutions, activation functions, dropout, and normalization layers, followed by an average pooling and finally a linear classifier or regressor. This structure is present in many modern models, including ResNet~\cite{he_deep_2016} and DenseNet~\cite{densenet}. In some cases, a dropout layer precedes the classifier, as in GoogleNet~\cite{googlenet} or EfficientNet~\cite{effnet}. Some variants use a multilayer perceptron as the classifier~\cite{krizhevsky_imagenet_2012,simonyan_very_2014}. The consistent element is the average pooling between the encoder and the classifier. This paper shows that this key operator renders these models structurally and mechanistically analogous to self-explainable models.\looseness=-1

Self-explainable models (SEMs) constitute a class of architectures designed to improve the interpretability and transparency of deep learning models by explicitly associating predictions with human-understandable concepts or prototypes. Here, we distinguish these two by defining concepts as frequent/relevant patterns in the training dataset that may possess varying degrees of semantic meaning, such as "a blue sky" or "a plane." Prototypes are vectors representing these concepts in the embedding space, and the similarity between input data and these prototypes provides concept-based explainability for predictions. The selection of concepts and prototypes is a critical consideration. Prototypes selected a priori~\cite{tcav} may introduce biases and fail to capture the full diversity of the dataset~\cite{celis2016fair}. Jointly learning prototypes during training comes at a cost of a complicated alternating training scheme and may negatively affect the performance~\cite{senn,protopnet,pantypes}. Additionally, because these architectures differ from the traditional classifiers they are intended to explain or replace, the interpretability insights obtained may not transfer seamlessly.\looseness=-1

This work advocates for computing prototypes after training. Indeed, the classic feedforward architecture remains highly versatile and efficient. A common formalization framework for $k$-means-based post-hoc explanations is introduced, covering three explanation locations: the classifier weights, the encoder's final output, and its intermediate feature activations. It encompasses the cases where  the CNN is reinterpreted as a SEM, the linear classifier is substituted with a $k$-means-based one as described in \cite{kmex}, and the multi-depth explanation of the encoder building \cite{boubekki_explaining_2025}. The latter is here further extended to produce gradient-free and concept-aligned feature attribution maps.

\section{Problem Definition and Notations}

Notational conventions follow \cite{goodfellow_deep_2016}: capital letters denote dimensions, dot products are written $\cdot$, and biases are omitted from classifiers without loss of generality. The classifier is represented by a matrix $\mC$ with columns $\vc_j$.\looseness=-1

A standard CNN-based classifier architecture comprises a CNN encoder $\enc: \R^Q \rightarrow \R^{R \times D}$, an average pooling operation $\avgpool: \R^{R \times D} \rightarrow \R^D$, and a linear classifier $\clf: \R^D \rightarrow \R^C$. The encoder output is typically a tensor of dimension $(W,H,D)$. However, for legibility, the first two dimensions are here flattened. The operations are formally defined as follows:\looseness=-1
\begin{align}\label{eq:cnn1}
\vx \xmapsto{\enc} \vh \xmapsto{\avgpool} \vz  \xmapsto{\clf} \vy,
\end{align}
where $\vx \in \R^Q$, $\vh \in \R^{R \times D}$, $\vz \in \R^D$ and $\vy \in \R^C$, and:
\begin{align}\label{eq:cnn2}
    \avgpool(\vh) = \dfrac{1}{R} \sum_{r=1}^R \vh_r = \vz \quad \text{and} \quad \clf(\vz) = \vz \cdot \mC = ( \vz \cdot \vc_j )_{j=1\ldots C}.
\end{align}

\paragraph{SEM} 
Self-explainable models differ from classic CNNs by their output layer. In this case, a classifier which compares a feature vector, also noted $\vh$, to a set of $K>0$ prototypes $\mP = \{ \vp_k \}_K$ using a similarity measure $\simm$. The similarity scores are then mapped using $\proj$ into a prediction score $\vy \in \R^C$. \looseness=-1

\begin{align}\label{eq:sem1}
\begin{split}
\vx \xmapsto{\enc } \vh \xmapsto{\simm}  \Big( \simm( \vh_r, \vp_k ) \Big)_{R,K} \xmapsto{\proj} \vy.\\
\end{split}
\end{align}

The dimensions of $\vh$ are intentionally omitted as the location of the prototypical classifier depends on which part of the CNN we want to explain.  The similarity measure varies between models and is usually implemented as either a dot product~\cite{flint} or a distance metric~\cite{protopnet}. In order to remain \emph{transparent}, $\proj$ should remain as simple as possible. Typically, it is either a matrix multiplication~\cite{protopnet} or a pooling operation.\looseness=-1

\paragraph{Explaining with ProtoPNet}
We use ProtoPNet introduced in \cite{protopnet} as the SEM baseline. It employs an alternating update strategy: the encoder is updated via gradient descent, and the prototype layer is updated via a multi-stage procedure. The similarity measure is $\ell_2$-based, and $\proj$ is a linear layer.

\section{Post-hoc Self-explanation of CNNs}

This section presents a common formalization for three $k$-means-based post-hoc self-explanation of frozen CNNs.

\subsection{Self-explaining the Classifier}
A CNN can be reinterpreted as a SEM. Indeed, since the classifier is a linear layer, its column vectors can be viewed as prototypes, and the matrix operation as the similarity measure.\looseness=-1 

\begin{theorem} Convolutional neural network classifiers are self-explainable models with $C$ prototypes corresponding to the vector columns of the classifier. 
\end{theorem}

\begin{proof} 
Consider a CNN classifier as defined in Equations~\ref{eq:cnn1} and \ref{eq:cnn2}.
The commutativity of the average pool and dot-product operations means that the final prediction is also the average prediction of each $\vh_r$:
{\begin{align}\label{eq:proof2}
    \clf \circ \avgpool (\vh) = \bigg( \Big(\dfrac{1}{R} \sum_{r=1}^R \vh_r \Big) \cdot \vc_j \bigg)_{j=1\ldots C} = \bigg( \dfrac{1}{R} \sum_{r=1}^R \big( \vh_r \cdot \vc_j \big) \bigg)_{j=1\ldots C}
\end{align}}
By setting $\proj = \avgpool$, defining $\simm$ as the dot-product, and treating $\vc_j$ as prototypes, the following is obtained:
\begin{align}\label{eq:proof2}
    \clf \circ \avgpool \circ \enc (\vx) = \dfrac{1}{R} \sum_R ( \vh_r \cdot \vc_j ) = \proj \circ \simm \circ \enc (\vx)
\end{align}
Therefore, the operations of the CNN can be reinterpreted as those of an SEM, as defined by Equation~\ref{eq:sem1}.
\end{proof}

In practice, the $\vc_j$ vectors are not satisfactory prototypes, particularly with respect to diversity. Their number is inherently limited by the number of classes. Preliminary experiments indicate that simply increasing the number of classes and merging them using max or average pooling, without additional regularization, results in a single direction dominating each class. 
Although the cross-entropy loss encourages prediction separability, leading to class embeddings covering distinct half-spaces, it does not guarantee that points are centered or distributed along the line defined by the prototype or column vector, meaning that prototypes may be far from the data. This contradicts the implicit requirement of representativeness or diversity of the prototypes.\looseness=-1

\subsection{Explaining the Classifier}
A workaround to the lack of representativeness is to replace the classifier with a $k$-means based one as introduced in \cite{kmex}. As the centroids of class-wise clusterings, the prototypes are thus more likely to be closer to the data. The procedure is as follows:\looseness=-1

1. For each of the $C$ classes, $K/C$ prototypes are learned using $k$-means on $\vz$ of the training data.\\ \looseness=-1
2. The similarity measure is the exponential of minus the $\ell_2$ distance: $\simm(\vh, \vp_k) \!=\! \exp({\text{-}||\vh, \vp_k||^2})$.\\ \looseness=-1
3. The $\proj$ returns a one-hot vector centered on the class of the prototype with largest similarity score.\looseness=-1

The original KMEx utilizes the $\ell_2$ distance as a similarity measure and employs a nearest neighbor classifier. The similarity described in the second step is equivalent and allows to derive the class predictions using an $\argmax$ . Although the method can accommodate varying numbers of prototypes per class, we fix them to $K/C$ prototypes per class.\looseness=-1

\subsection{Explaining the Encoder}
To explain the encoder rather than the classifier, intermediate outputs, also referred to as feature activations, are compared to prototypes instead of the embedding vector after average pooling, like in the previous section. This approach aims to access information that may be filtered out before reaching the classifier and provide explanations at the patch or segment level.
ProtoPNet is an example of such an approach, as its so-called \emph{part} prototypes are compared to the pixels of the encoder's output. Other models~\cite{flint,zhu2025interpretable} utilize shallower feature activations to compute part prototypes. In these cases, the prototypes are composite representations that integrate information from multiple depths. The post-hoc framework presented here follows the rationale of KMEx~\cite{kmex} by using $k$-means to learn prototypes on a frozen backbone. Formally, it extends the work in \cite{boubekki_explaining_2025} to compute predictions and feature attributions.\looseness=-1

\paragraph{Extracting Feature Activations}
Let us first decompose the encoder into $B$ blocks of layers:
\begin{align}
    \vx \xmapsto{f_1} \vx^{(1)} \cdots \xmapsto{f_b} \vx^{(b)} \cdots \xmapsto{f_B} \vx^{(B)}=\vh, \quad \text{where} \quad \vx^{(b)} \in {\R}^{ R_b \times D_b}.
\end{align}
Since the outputs of each block have different resolutions $R_b$ and numbers of channels or dimensions $D_b$, preprocessing steps are required to compute a composite matrix $\widecheck{\vh}$. This matrix is then compared to the prototypes to generate a prediction vector $\vy$. The $K/C$ class prototypes are learned using $k$-means on the row vectors $\widecheck{\vh}_r$ of the class data. The operations are as follows.\looseness=-1

1. The intermediate outputs are first linearly interpolated to a shared resolution $R'$.
 \begin{align}
 \operatorname{Upsample} \big( \vx^{(b)} \big) = \vu^{(b)} \in {\R}^{R' \times D_b}.
\end{align}
2. The $\vu^{(b)}$ are then normalized and scaled before being concatenated into $\widecheck{\vh}$
\begin{align}
\operatorname{Concatenate} \left( \dfrac{\vu^{(b)}}{D_b ||\vu^{(b)}||}  \right) = \widecheck{\vh} \in {\R}^{R' \times (\sum_b D_b)}
\end{align}
3. The assignments of the rows $\widecheck{\vh}_r$ to the closest prototypes are stored in a \emph{binary matrix}:
\begin{align}
     \simm\big( ( \vx^{(b)} )_B, \mP \big)  = \simm\big( \widecheck{\vh}, \mP \big)  = \argmin_k( || \widecheck{\vh}_r - \vp_k|| ) = \widecheck{\vs} \in {\R}^{R' \times K}
\end{align}
4. The prediction vector is computed as the average count of the class-wise clusters in $\widecheck{\vs}$, and the predicted class corresponds to the most frequently occurring cluster.
\begin{align}
    \proj( \widecheck{\vs} ) = \avgpool \left( \sum_{r} \widecheck{\vs}_r \:; K/C \right) = \vy \in {\R}^C
\end{align}
The average pooling operation, which uses $K/C$ non-overlapping windows, assumes a constant and ordered number of clusters per class.
The assignment binary matrix $\widecheck{\vs}$ can be interpreted as a low-resolution segmentation of the input, referred to as an ``\emph{explanation map}''. In practice, outputs from all blocks after a specified depth are combined. Accordingly, the notations $\widecheck{\vh}^{(b:)}$ and $\widecheck{\vs}^{(b:)}$ indicate that all outputs from $f_b$ to $f_B$ are utilized.\looseness=-1

\paragraph{Feature Importance}

Numerous feature attribution methods have been proposed~\cite{IG,GC,RISE}, resulting in a wide variety of outputs. The present work introduces a simplified approach that leverages the self-explainable reinterpretation of a CNN and the commutativity of average pooling with respect to the dot product. To keep track of the input format, this section explicitly indicates  the width $W_b$ and height $H_b$ of the outputs, rather than the aggregated $R_b$ dimension.\looseness=-1

Analogous to the class activation map (CAM)~\cite{cam}, the feature attribution score of pixel $\vh_{wh}$ is defined as a measure of its alignment with $\vc_j$.\looseness=-1
\begin{align}
    \att ( \vh_{wh}, \vc_j ) = \dfrac{ \vh_{wh} \cdot \vc_j }{|| \vc_j||^2}
\end{align}
The distinction lies in the normalization by the squared norm, which ensures that $\att( \vc_j^T, \vc_j )=1$ and allows values to be compared across classes.\looseness=-1

The resulting attribution map matches the low resolution of the encoder's output. Rather than relying on assumptions about the information conveyed by gradients, as in Grad-CAM variants~\cite{GC}, the upstream feature attribution map is approximated by exploiting the spatial consistency of the convolutional receptive fields.
Namely, the attribution map at depth $b$ is computed from the one at depth $b+1$ as follows:
\begin{enumerate}
    \item Compute the explanation map at depth $b$, $\widecheck{\vs}^{(b:)} \in {\R}^{W_b \times H_b \times K}$. 
    \item Upsample $\att^{(b+1:)}  \in {\R}^{W_{b+1} \times H_{b+1}} $ to $\att^{(b+1:,\text{up})} \in {\R}^{W_b \times H_b}$.
    \item Compute $\att^{(b:)}$ as the segment-wise average of $\att^{(b+1:,\text{up})}$ based on $\widecheck{\vs}^{(b:)}$.
\end{enumerate}
The segment-wise averaging results in discrete attribution, where pixels within the same segment share an identical attribution score.\looseness=-1

\begin{figure}[!h]
	\centering
	\hfill
    \begin{tabular}{c}
        \small{B4} \\
	    \includegraphics[width=.4\linewidth]{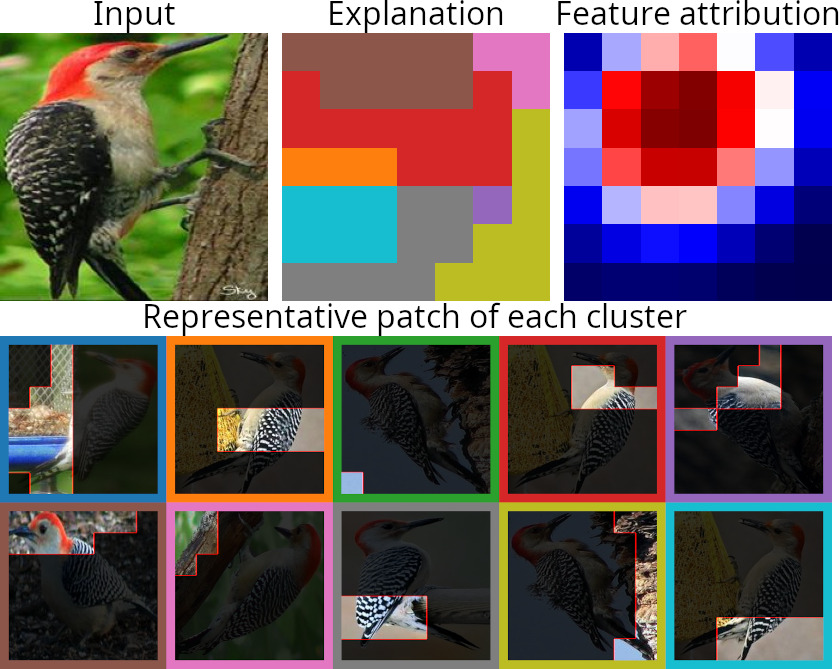} 
    \end{tabular}
	\hfill
    \begin{tabular}{c}
        \small{B234} \\
	    \includegraphics[width=.4\linewidth]{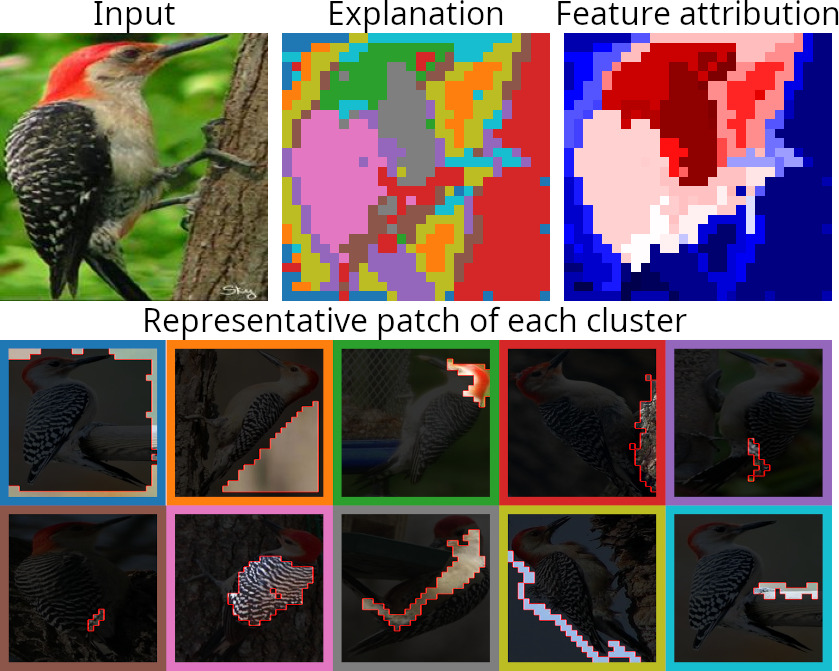} 
    \end{tabular}
	\hfill
	\caption{Interpretation process for B4 (left) and B234 (right) on a CUB-200 red-bellied woodpecker. Red and blue indicate higher and lower feature importance. Representative patches are the closest training examples to each prototype; border colors match the explanation map segments.\looseness=-1}
	\label{fig:interp}
\end{figure}

\paragraph{Interpretation Process}
Figure~\ref{fig:interp} illustrates the interpretation process of the $k$-means based explanation of the encoder. Each input receives two forms of explanation: a segmentation map and a feature importance map. The segmentation map delineates regions of the input that yield similar feature activations, resulting in clustered segments. Cluster interpretation is supported by visualization of the nearest segment to associated prototype, using a consistent color scheme. The feature importance map highlights the relevance of each concept with respect to the classifier of the backbone model. Incorporating shallower feature activations enhances the resolution and detail of the map. \looseness=-1

\section{Experiments}

\paragraph{Experimental Setting} 
All experiments use a ResNet34 backbone pretrained on ImageNet, trained following \cite{kmex}, and evaluated on MNIST~(\cite{mnist}), STL10~(\cite{stl10}), and CUB-200~(\cite{WahCUB_200_2011}), with five prototypes per class for the first two and ten for CUB-200, if not specified otherwise.
Our method is referenced according to the depth of the feature activations employed. The ResNet34 backbone consists of a series of preprocessing layers and four residual blocks. The model utilizing the same information as ProtoPNet, specifically at the encoder's output (the fourth block), is designated as B4. The model B234, which incorporates outputs from the last three blocks, is also evaluated. We compare these to KMEx and ProtoPNet.\looseness=-1

\paragraph{Alignment of the Prototypes}
The column vectors $\mathbf{c}_j$ of the classifier define the half-space in which the class embedding predominantly resides, rather than the direction along which the data spreads. Table~\ref{tab:cos} validates this interpretation using the average cosine similarity between each method's prototypes and their class embedding (\emph{class}) and all other data points (\emph{out}).

The uniformly high cosines of ProtoPNet reflect the distinct embedding geometry produced by an integrated training. The \emph{class} cosines of the CNN's classifier weights are approximately $0.5$, versus nearly $1$ for KMEx, indicating poor alignment of the $\mathbf{c}_j$ with the data. The \emph{out} cosines suggest that the $\mathbf{c}_j$ define nearly orthogonal half-spaces, though the class embeddings themselves are not fully orthogonal. Figure~\ref{fig:umap} illustrates this misalignment via a UMAP projection~\cite{umap} of the CNN's embeddings for the twenty sparrow classes of CUB-200. The classifier prototypes $\mathbf{c}_j$, both raw (crosses) and rescaled to class norm (triangles), are dispersed and rarely near their respective class; by contrast, the KMEx prototypes (squares) lie within their respective clusters.\looseness=-1

\begin{figure}[!h]
	\centering
	\hfill
	\includegraphics[width=.45\linewidth]{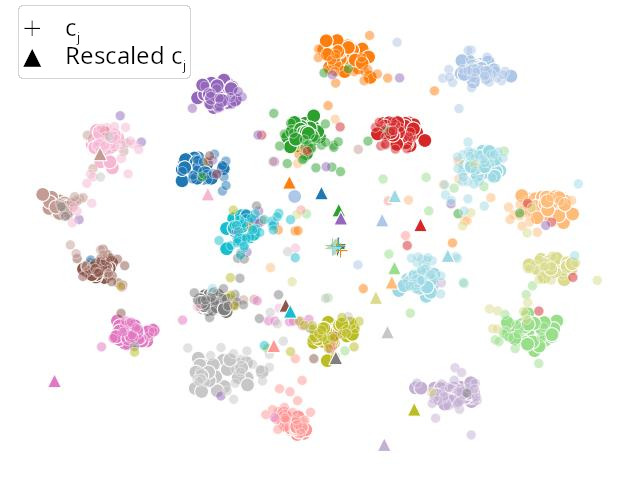}
	\hfill
	\includegraphics[width=.45\linewidth]{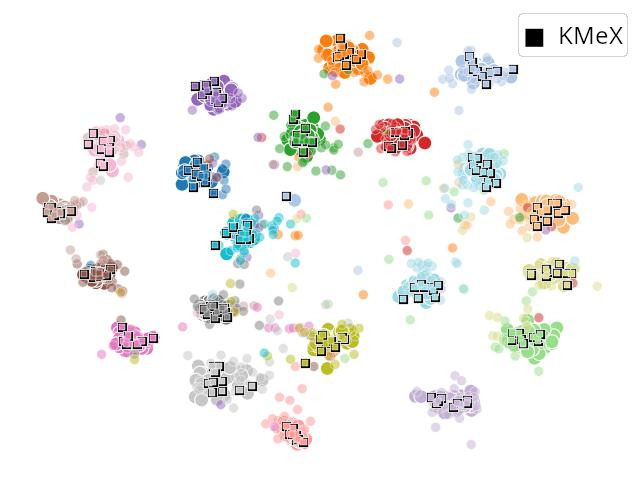}
	\hfill
	{\caption{UMAP projection of training and test (lower opacity) embeddings from the twenty sparrow classes of CUB-200. Left: classifier prototypes $\mathbf{c}_j$ as crosses and rescaled to class norm as triangles. Right: KMEx prototypes.}
	\label{fig:umap}}
\end{figure}

\begin{table}
\centering
\begin{minipage}{0.5\linewidth}\centering
  \caption{Average cosine of classifier prototypes within \emph{class} and \emph{out} of class embedding.}
  \label{tab:cos}
{   \small \centering
          \begin{tabular*}{\linewidth}{@{\extracolsep{\fill}}lrlrlrl}
                  \toprule
                       & \multicolumn{2}{c}{MNIST} & \multicolumn{2}{c}{STL10} & \multicolumn{2}{c}{CUB200} \\ 
Model                  & class & out & class & out & class & out \\ \midrule
CNN                    & $0.48$ & $-0.05$ & $0.46$ & $-0.04$ & $0.51$ & $0.00$ \\ 
KMEx {\tiny($K\!/\!C\text{=}1$)}             & $0.88$ & $0.32$  & $0.79$ & $0.42$  & $0.93$ & $0.47$ \\ 
KMEx {\tiny($K\!/\!C\text{=}5$)}   & $0.83$ & $0.30$  & $0.84$ & $0.34$  & $0.87$ & $0.45$ \\ 
ProtoPNet              & $1.00$ & $1.00$  & $0.95$ & $0.92$  & $0.98$ & $0.97$ \\
                  \bottomrule \vspace*{1.7em}
          \end{tabular*}}
\end{minipage}
\hfill
\begin{minipage}{0.48\linewidth}\centering
  \caption{Average train and test accuracies over five runs. Models were pretrained on ImageNet.}
  \label{tab:acc}
  {   \small \centering
          \begin{tabular*}{\linewidth}{@{\extracolsep{\fill}}lrlrlrl}
                  \toprule
          & \multicolumn{2}{c}{MNIST} & \multicolumn{2}{c}{STL10} & \multicolumn{2}{c}{CUB200} \\
Model            & train & test & train & test & train & test \\ \midrule
CNN              & $100.0$ & $99.4$ & $95.6$ & $86.1$ & $100.0$ & $79.1$ \\ 
KMEx {\tiny($K\!/\!C\text{=}5$)}     & $100.0$ & $99.4$ & $94.9$ & $85.6$ & $100.0$ & $78.6$ \\ 
ProtoPNet        & $100.0$ & $99.4$ & $74.2$ & $72.7$ & $99.4$ & $66.4$ \\
\multicolumn{4}{l}{\hspace{-2pt}{ProtoPNet in \cite{protopnet}} }   &  &  & $79.2$ \\ \midrule

Ours B4     & $100.0$ & $99.4$ & $94.0$ & $85.5$ & $100.0$ & $75.2$ \\ 
Ours B234   & $98.9$ & $98.5$ & $83.5$ & $77.9$ & $85.5$ & $52.8$ \\ 
                  \bottomrule
          \end{tabular*}}
\end{minipage}
\end{table}

\paragraph{Concept-based Accuracy}
Regarding train and test accuracies (Table~\ref{tab:acc}), replacing the linear classifier with KMEx does not affect the performance. This outcome is expected, since the prototypes of KMEx are learned class-wise and based on the same information as the $\vc_j$. 
Although B4 derives the class predictions from the concept proportions, it matches the performance of the backbone network (Table~\ref{tab:acc}). However, including information at shallower depths, as in B234, renders the classification signal less distinct, leading to reduced performance.\looseness=-1

\section{Discussion and Related Work}
Concept-based explainability methods for deep networks have been extensively reviewed in \cite{lee2025concept}. The proposed framework offers a common formalization of $k$-means-based post-hoc explanations, explicitly avoiding gradient- or perturbation-based techniques. We position our work with respect to the most closely related methods for explaining the encoder. TCAV~(\cite{tcav}) probes feature activations at a single depth using user-defined concept vectors and assigns feature importance using directional derivatives of the classifier's output with respect to the concept vectors. Its extension, ACE~(\cite{ghorbani2019towards}), automates concept discovery by segmenting inputs with SLIC~(\cite{SLIC}) and processing the resulting crops through the encoder. Since SLIC segmentation operates independently of the encoder's learned representations, it explains image regions rather than internal concepts. Additionally, projecting crops through an encoder trained on complex images is likely to result in out-of-distribution behavior. Substituting TCAV with SHAP produces CONE-SHAP~(\cite{li2021instance})), which retains the same limitations. Network Dissection~(\cite{bau2017network}) and Net2Vec~(\cite{fong2018net2vec}) probe individual neurons at a single depth to identify human-interpretable concepts. For the final layer of a ResNet34 (B4), this involves comparing 512 neurons, which, as reported by the authors, are often redundant. CRAFT~(\cite{fel2023craft}) and ICACE~(\cite{zhang2021invertible}) apply matrix factorization to feature activations to extract concepts, and leverage implicit differentiation and non-negative constraints, respectively, to produce concept attribution maps. However, restricting the back-propagated signal to a single class leaves portions of the input unexplained.\looseness=-1

\section{Conclusion}

This paper introduces a common formalization of $k$-means-based post-hoc explanations covering three locations: the classifier weights, the encoder's final output, and its intermediate feature activations. The commutativity of average pooling and the linear classifier shows that CNNs are mechanistically analogous to prototype-based SEMs. However, the classifier's column vectors, when interpreted as prototypes, are poorly aligned with the data, undermining their representativeness. Replacing them with $k$-means centroids addresses this limitation without loss of performance. Extending this to intermediate feature activations produces detailed, dense explanation maps supported by semantically relevant prototypes. Attribution maps are computed via a gradient-free variant of CAM that leverages the spatial consistency of convolutional receptive fields. Together, these results suggest that the perceived black-box nature of CNNs is not an inherent property, but one that can be mitigated through rigorous reinterpretation of their existing operations.

The primary limitation is the scalability of $k$-means to large datasets, though preliminary results suggest that performance remains stable on subsets of the training data. The sustained classification performance of B4 and B234 using cluster distributions calls for further investigation.\looseness=-1

\bibliographystyle{iclr2026_conference}
\bibliography{references}

\end{document}

%% file: math_commands.tex
\usepackage{amsmath,amsfonts,bm}

\def\eqref#1{equation~\ref{#1}}

\def\1{\bm{1}}

\def\vc{{\bm{c}}}

\def\vh{{\bm{h}}}

\def\vp{{\bm{p}}}

\def\vs{{\bm{s}}}

\def\vu{{\bm{u}}}

\def\vx{{\bm{x}}}
\def\vy{{\bm{y}}}
\def\vz{{\bm{z}}}

\def\mC{{\bm{C}}}

\def\mP{{\bm{P}}}

\DeclareMathAlphabet{\mathsfit}{\encodingdefault}{\sfdefault}{m}{sl}
\SetMathAlphabet{\mathsfit}{bold}{\encodingdefault}{\sfdefault}{bx}{n}

\newcommand{\R}{\mathbb{R}}

\DeclareMathOperator*{\argmax}{arg\,max}
\DeclareMathOperator*{\argmin}{arg\,min}